\documentclass[a4paper,11pt]{article}
\usepackage{amsthm}
\usepackage{algorithm2e}
\usepackage{color}
\usepackage{graphicx}
\usepackage{makeidx}
\usepackage{amsmath}
\usepackage{amssymb}
\usepackage{float}
\usepackage{amssymb}
\usepackage{float}

\begin{document}

\newtheorem{theorem}{Theorem}
\newtheorem{lemma}{Lemma}
\newtheorem{definition}{Definition}

\title{\textbf{Removing Algorithmic Discrimination \\(With Minimal Individual Error)}}
\author{El Mahdi El Mhamdi, Rachid Guerraoui, L\^{e} Nguy\^{e}n Hoang and Alexandre Maurer
\\EPFL
\\elmahdi.elmhamdi@epfl.ch, rachid.guerraoui@epfl.ch,
\\le.hoang@epfl.ch, alexandre.maurer@epfl.ch}

\date{}
\maketitle

%\bibliographystyle{plain}
%\bibliography{biblio}

\begin{abstract}
We address the problem of correcting group discriminations within a score function, while minimizing the individual error. Each group is described by a probability density function on the set of profiles. We first solve the problem analytically in the case of two populations, with a uniform bonus-malus on the zones where each population is a majority. We then address the general case of $n$ populations, 
where the entanglement of populations does not allow a similar analytical solution.
We show that an approximate solution with an arbitrarily high level of precision can be computed with linear programming. Finally, we address the inverse problem  where the error should not go beyond a certain value and we seek to minimize the discrimination.
\end{abstract}

\section{Introduction}

As machine learning is being deployed, a growing number of cases of discriminatory behaviors is being highlighted. In 2016,   a  study by ProPublica\footnote{See https://tinyurl.com/machine-bias-sentencing} showed that some algorithmic assessment of recidivism risks was significantly racially biased against black criminals. Indeed, 45\% of supposedly high-risk black criminals did not re-offend, as opposed to 22\% of supposedly high-risk white criminals. Conversely, 28\% of supposedly low-risk black criminals re-offended, as opposed to 48\% of supposedly low-risk white criminals. Such concerns for algorithmic discrimination have fostered a lot of work.

A major difficulty posed by new machine learning techniques is that algorithms may have learned their biases from high-dimensional data, which ironically seems hard to handle without machine learning. Racial inequalities in facial recognition have for instance been showed in \cite{buolamwini2018}. More disturbingly, it was discovered that the popular word2vec package \cite{mikolov2013} yields gender discriminative relations between word representations, e.g., $doctor-man+woman=nurse$. In other words, word2vec seems to infer from natural language processing that a man is to a woman what a doctor is to a nurse. Evidently, this is only one example out of many. Such examples illustrate the difficulty of mitigating algorithmic discrimination.

Many solutions have been proposed. Some consist in pre-processing data used for machine learning \cite{dxp2, dxp4, dxp5, dxp6, dxp15} or making it unbiased \cite{dxp7}. Some try to prevent discrimination during the learning phase \cite{dxp13,dxp14,dxp3}, by using causal reasoning \cite{dxp1}, or with graphical dependency models \cite{dxp16}.
Other approaches try to achieve independence from specific sensitive
attributes \cite{dxp8,dxp18,dxp17}.
Dwork et al. \cite{dxp19} introduced the concept of ``fair affirmative action'', to improve the treatment of specific groups while treating similar individuals similarly.
Algorithmic discrimination was also considered in problems of subsampling \cite{dxp9}, voting \cite{dxp10}, personalization \cite{dxp11} or ranking \cite{dxp12}.

All previous works highlighted the fundamental trade-off between \emph{group discrimination} (i.e., some groups being globally penalized compared to other groups) and \emph{individual accuracy} (i.e., individuals being judged with a high level of precision). In this paper, we propose a post-processing approach to remove group discrimination while minimizing the individual error\footnote{The social impact of removing group discrimination and the extent to which such an enterprise is desirable are out of the scope of this paper. We ``simply'' address the problem of doing it with a minimal error.}, as well as an approach to minimize group discrimination given an individual error constraint. 

%Post-processing approaches enable to treat any function obtained by machine learning as a ``black box'', without going through the learning phase all over again. Post-processing can also be used in addition to other techniques (e.g., if important discriminations remain) or when other techniques are not applicable (e.g., when sensitive and non-sensitive attributes cannot be clearly separated).

More specifically, we assume that we are given a score function $f$ that computes a score $f(x)$ for each individual $x$. Here, the individual's profile $x \in S$ can be any sort of description of the individual. In simple settings, it may be a collection of real-valued features, i.e. $S = \mathbb R^n$, and the scoring function $f$ may be interpretable. However, as machine learning improves, rawer data are being used to score individuals, e.g. they may be textual biographies of undetermined length. In such cases, the scoring function $f$ is usually constructed via machine learning, and it often has to be regarded as some ``black box''. To remove group discrimination, rather than pre-processing raw data or modifying the learning phase, it may thus be simpler to perform some post-processing of the score function, i.e. deriving a non-discriminative score function $h$ from the possibly discriminative function $f$.

An additional difficulty is that the individual's profile $x$ may not clearly determine its sensitive features, e.g. gender or race. Nevertheless, evidently, even biographic texts may provide strong indications of the individual's likely sensitive features. A natural approach to analyze the dependency of the score function on sensitive features is to test its scoring on profiles that are representative of a certain gender or race. Interestingly, this approach can now be simulated using so-called generative models \cite{goodfellow2014,karras2017}. These models allow to draw representative samples of subpopulations of individuals. 

Thus, we assume that any population $i$ (women, men, black, white, \dots) can be described by some generative model. Formally, this corresponds to saying that the population $i$ is represented by a probability density function $p_i$ on $S$. Given $p_i$, we can determine the \emph{average score} of population $i$ (i.e., $\int_{x \in S} p_i(x) f(x) dx$), which can be well approximated by sampling the generative model associated to population $i$. A toy example of average score is given in Figure~\ref{fig:distrib}.

\begin{figure}
\begin{center}
\includegraphics[width=14cm]{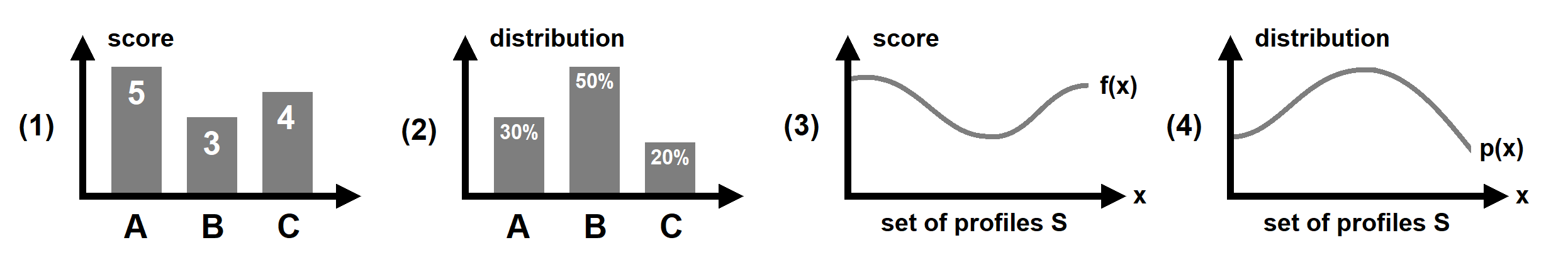}
\caption{Illustration of the average score of a given population. First, we consider a set of profiles $S = \{A,B,C\}$ (only $3$ possible profiles) in plots (1) and (2). Plot (1) represents the score associated with each profile, and plot (2) represents the fraction of the population associated with this profile. Thus, the average score of the population is $0.3 \times 5 + 0.5 \times 3 + 0.2 \times 4$. Plots (3) and (4) are a continuous version of plots (1) and (2). Here, the average score is $\int_{x \in S} p(x) f(x) dx$.} 
\label{fig:distrib}
\end{center}
\end{figure}

{\bf Contributions.} We first study in this paper the simple case of two populations with a different average score.
The goal here is to determine a new score function $h$ where (a) the two populations have the same average score and (b) the \emph{individual error} is minimized. We define the individual error as the maximal difference between $f$ and $h$, i.e., $\max_{x \in S} |f(x) - h(x)|$ (also written $|| f - h ||_\infty$).
We call the problem of determining the best function $h$ the 2-ODR (2-Optimal Discrimination Removal) problem (``2'' standing for ``two populations'').
We  present an exact solution to the 2-ODR problem.
Roughly speaking, we consider the subsets of $S$ where $p_1(x) > p_2(x)$ and $p_2(x) > p_1(x)$,
and apply a uniform bonus (or penalty) on these subsets.
We show that our solution is indeed optimal for it minimizes the individual error. % Our solution leverages the fact that we restrict ourselves to two populations.
%To our knowledge, this is the first paper to use $|| . ||_\infty$ as a metric in the context of fair affirmative action.

Then we turn to the more general case of $n$ populations, which is arguably the most relevant setting in practice. Indeed, it is for instance often considered important that a score function be both non-racist {\it and} non-sexist. Similarly, it may be relevant to compare the scores of several races, e.g. Black, White, Asian and Arabic. In fact, we may even demand greater granularity by also comparing black female and white female, in addition to already comparing black and white. 
We address this $n$-population setting by considering some desired average score $y_i$ for each population $i$.
%We then consider the more difficult problem of $n$ populations $(n \geq 2)$. 
%Besides, instead of aiming at a same score for \emph{all} populations, we consider an arbitrary \emph{desired average score} $y_i$ for each population $i$
This more general goal enables the modelers to describe more subtly what they consider desirable. We call this problem the
Optimal Discrimination Removal (ODR) problem.

This problem is significantly more difficult with $n > 2$. In fact, we conjecture that it is computationally intractable for $n>2$ and combinatorially large profile sets $S$. Indeed, intuitively, in the case $n=2$, the general problem of removing discrimination could be fixed locally for each $x \in S$, by determining whether $x$ is more likely to be of population 1 or 2. Unfortunately, this no longer seems to be the case when $n>2$. To solve the ODR problem, it seems that a global solution $h$ first needs to be derived. But this global solution seems to require at least $\Omega(|S|)$ computation steps in general.

Interestingly though, we show that an approximate solution (with an arbitrarily high level of precision) can be obtained with \emph{linear programming} \cite{zlp0}. Linear programming problems are expressed in terms of a set of inequalities involving linear combinations of variables. These problems have been extensively studied, and a lot of algorithms have been proposed to solve them \cite{zlp1,zlp2,zlp5,zlp6}. Here, we show that this abundant literature of algorithms can also be leveraged to solve discrimination problems.

We proceed incrementally through 6 steps. We first show that the ODR problem is reducible to the simpler (to express)
Optimal Bonus-Malus (OBM) problem, where each desired average score is $0$.
We then define an approximate version of 
OBM, which we denote AOBM.
We consider an arbitrary partition $(S_1,\dots,S_m)$ of $S$,
as well as a set of functions $Z$ which are ``flat'' on each subset $S_j$. The AOBM problem consists in approximating a solution to the OBM problem with a function $u \in Z$. The larger $m$, the more precise the solution.
We show that the AOBM problem is equivalent to a linear programming problem with $2m+1$ variables and $m+2n$ inequalities\footnote{Excluding the inequalities requiring each variable to be positive (which are included in the canonical form of a linear programming problem).}. We use the fact that the functions of $Z$ can only take a finite number of values, to transform the continuous OBM problem into a discrete problem.

We finally also address the inverse problem, where the individual error is not allowed to be greater than $\epsilon$. Here, the goal is to be \emph{as close as possible} to the desired score of each population. We proceed in an analogous way through 6 steps.

The case of two populations is treated in Section~\ref{sec_2pop}, the general case in Section~\ref{sec_gen},
and the inverse case in Section~\ref{sec_altpb}.
We conclude in Section~\ref{sec_conc}.

\section{The Case of Two Populations}
\label{sec_2pop}

Let $S$ be a set of \emph{profiles}.
%For instance, assuming $N$ criteria $c \in \mathbb{R}$, then $S = \mathbb{R}^N$.
Let $f$ be a function from $S$ to $\mathbb{R}$ associating a \emph{score} to each profile.
Let $p_1$ and $p_2$ be any two probability density functions on $S$, representing two populations $1$ and $2$.

Let $X$ be the set of functions $g$ from $S$ to $\mathbb{R}$ such that
$\int_{x \in S} p_1(x) g(x) dx = \int_{x \in S} p_2(x) g(x) dx$
(i.e. population $1$ and $2$ have the same average score).

For any function $g$ from $S$ to $\mathbb{R}$, let
$|| g ||_\infty = \max_{x \in S} |g(x)|$.

The 2-ODR (2-Optimal Discrimination Removal) problem consists in finding a function 
$h \in \arg \min_{g \in X} || g - f ||_\infty$, i.e., a function minimizing the individual error.

\paragraph{Solution.}

For $x \in S$, let $u(x) = 1$ if $p_1(x) > p_2(x)$ and $-1$ otherwise.

Let $A = \int_{x \in S} (p_1(x) - p_2(x)) u(x) dx$, 
$B = \int_{x \in S} (p_1(x) - p_2(x)) f(x) dx$
and $k = -B/A$.

We define $h$ by $h(x) = f(x) + k u(x)$.

\bigskip

\begin{theorem}
Function $h$ above solves the 2-ODR problem.
\end{theorem}

\begin{proof}
By construction, $|| h - f ||_\infty = |k|$.
If $k = 0$, $h$ indeed minimizes $|| h - f ||_\infty$. We now suppose that $k \neq 0$.

The proof is by contradiction.
Suppose the opposite of the claim: there exists a function $h' \in X$
such that $|| h' - f ||_\infty < |k|$.
Then, $h'(x) = f(x) + v(x)$, with $|v(x)| < |k|$.
Let $D = \int_{x \in S} p_1(x) h'(x) dx - \int_{x \in S} p_2(x) h'(x) dx $.
Then,
$D =  \int_{x \in S} (p_1(x) - p_2(x)) f(x) dx  +  \int_{x \in S} (p_1(x) - p_2(x)) v(x) dx$.

By definition, $k = -B/A$, with
$A = \int_{x \in S} (p_1(x) - p_2(x)) u(x) dx$ and
$B = \int_{x \in S} (p_1(x) - p_2(x)) f(x) dx$.
Thus, $B = -kA$, and
$D = -k \int_{x \in S} (p_1(x) - p_2(x)) u(x)dx  +   \int_{x \in S} (p_1(x) - p_2(x)) v(x) dx
= \int_{x \in S} (p_1(x) - p_2(x)) ( v(x) -k u(x)) dx$.

Let $S_1$ (resp. $S_2$) be the subset of $S$ such that $p_1 > p_2$ (resp. $p_2 < p_1$).
Then, $D = D_1 + D_2$,
where
$D_i = \int_{x \in S_i} (p_1(x) - p_2(x)) ( v(x) -k u(x)) dx$.
Let $s(x) = 1$ if $x > 0$ and $-1$ otherwise.

If $p_1(x) > p_2(x)$ (resp. $p_2(x) < p_1(x)$), $u(x) = 1$ (resp. $u(x) = -1$).
Then, as $|v(x)| < |k|$, we have
$s(v(x) - k u(x)) = -s(k)$
(resp. $s(k)$).
Thus, $s(D_1) = s(D_2) = -s(k)$,
and $D = D_1 + D_2 \neq 0$.

Therefore,
$\int_{s \in S} p_1(x) h'(x) dx \neq \int_{s \in S} p_2(x) h'(x) dx$,
and $h' \notin X$:
contradiction. Hence, our result.
\end{proof}

\section{The General Case}
\label{sec_gen}

We now consider the case of $n$ populations.
This problem when $n \geq 2$ is significantly harder than the problem above due to the entanglement of several probability density functions.
We show that an approximate solution of this problem can be obtained with linear programming. We proceed incrementally through 6 steps.

\begin{enumerate}

\item We define the general Optimal Discrimination Removal (ODR) problem, corresponding to the case $n \geq 2$.

\item We define a simpler (to express) problem, the Optimal Bonus-Malus (OBM) problem.

\item We show that solving OBM provides an immediate solution to ODR.

\item We define an approximate version of the OBM problem (AOBM), where we restrict ourselves to functions which are ``flat'' on an arbitrarily large number of subsets of $S$.

\item We define a Linear Programming problem, that we simply call LP for convenience.

\item We show that LP also solves AOBM.

\end{enumerate}

\subsection*{Step 1: The Optimal Discrimination Removal (ODR) Problem}

Let $(p_1,p_2,\dots,p_n)$ be $n$ probability density functions on $S$, each one representing a population.
Let $(y_1,y_2,\dots,y_n)$ be $n$ arbitrary values.

Let $\Omega_0$ be the set of functions $g$ from $S$ to $\mathbb{R}$ such that, $\forall i \in \{1,\dots,n\}$,
$\int_{x \in S} p_i(x) g(x) dx = y_i$
(i.e. the mean score of population $i$ is $y_i$).

If $\Omega_0 \neq \emptyset$, the ODR problem consists in finding a function 
$h \in \arg \min_{g \in \Omega_0} || g - f ||_\infty$.

\subsection*{Step 2: The Optimal Bonus-Malus (OBM) Problem}

$\forall i \in \{1,\dots,n\}$,
let $b_i = y_i - \int_{x \in S} p_i(x) f(x) dx$.

Let $\Omega$ be the set of functions $g$ from $S$ to $\mathbb{R}$ such that, $\forall i \in \{1,\dots,n\}$,
$\int_{x \in S} p_i(x) g(x) dx = b_i$.

If $\Omega \neq \emptyset$, the OBM problem consists in finding a function 
$u \in \arg \min_{g \in \Omega} || g ||_\infty$.

\subsection*{Step 3: Reducing ODR to OBM}

Theorem~\ref{th_eqivv} below says that a solution to the OBM problem provides an immediate solution to the ODR problem.

\begin{theorem}
\label{th_eqivv}
If $u$ solves the OBM problem, then $h = f + u$ solves the ODR problem.
\end{theorem}

\begin{proof}
As $u$ solves the OBM problem, we have the following:
$\forall g \in \Omega$,
$|| u ||_\infty \leq || g ||_\infty$.

Note that, if $g \in \Omega_0$,
then $g - f \in \Omega$.
Indeed, if $g \in \Omega_0$,
then $\forall i \in \{1,\dots,n\}$,
$\int_{x \in S} p_i(x) g(x) dx = y_i$.
Thus, $\forall i \in \{1,\dots,n\}$,
$\int_{x \in S} p_i(x) ( g(x) - f(x) ) dx = y_i - \int_{x \in S} p_i(x) f(x) dx = b_i$.
Thus, $g - f \in \Omega$.

Therefore, 
$\forall g \in \Omega_0$,
$|| u ||_\infty \leq || g - f ||_\infty$.
As $u = h - f$, we have:
$\forall g \in \Omega_0$,
$|| h - f ||_\infty \leq || g - f ||_\infty$.
Thus, $h \in \arg \min_{g \in \Omega_0} || g - f ||_\infty$. Thus, the result.
\end{proof}

\subsection*{Step 4: The Approximate OBM (AOBM) Problem}

Let $(S_1,\dots,S_m)$ be a partition of $S$:
$S_1 \cup S_2 \cup \dots \cup S_m = S$,
and $\forall \{i,j\} \in \{1,\dots,m\}$,
$S_i \cap S_j = \emptyset$.
Let $Z$ be the set of functions $z$ from $S$ to $\mathbb{R}$ such that,
$\forall i \in \{1,\dots,m\}$,
$\forall x \in S_i$
and $\forall x' \in S_i$,
$z(x) = z(x')$ (i.e. $z$ is ``flat'' on each subset $S_i$).

If $\Omega \cap Z \neq \emptyset$, the AOBM problem consists in finding a function $u \in \arg \min_{g \in \Omega \cap Z} || g ||_\infty$.

\subsection*{Step 5: The Linear Programming (LP) Problem}

Let $N$ and $M$ be two integers.
Let $(x_1,\dots,x_N)$ be $N$ variables.
Let $L$ and $(L_1,\dots,L_M)$ be $M+1$ linear combinations of the variables $(x_1,\dots,x_N)$.
Let $(c_1,\dots,c_M)$ be $M$ constant terms.

A \emph{linear programming} problem consists in finding values of $(x_1,\dots,x_N)$ maximizing L while verifying the following inequalities:

\begin{itemize}
\item $\forall k \in \{1,\dots,N\}$, $x_k \geq 0$

\item $\forall k \in \{1,\dots,M\}$, $L_k \leq c_k$ 
\end{itemize}

In the following, we define a specific linear programming problem, that we simply call LP problem for convenience.

$\forall i \in \{1,\dots,n\}$
and $\forall j \in \{1,\dots,m\}$,
let $v(i,j) = \int_{x \in S_j} p_i(x) dx$.

Let $(\alpha_1,\dots,\alpha_m)$,
$(\beta_1,\dots,\beta_m)$
and $\gamma$ be $2m+1$ variables.

Consider the following inequalities:

\begin{enumerate}

\item $\gamma \geq 0$,
and $\forall j \in \{1,\dots,m\}$,
$\alpha_j \geq 0$ and
$\beta_j \geq 0$.
\item $\forall j \in \{1,\dots,m\}$,
$\alpha_j - \gamma \leq 0$
and $\beta_j - \gamma \leq 0$.
\item $\forall i \in \{1,\dots,n\}$,
$\Sigma_{j=1}^{j=m} \alpha_j v(i,j) - \Sigma_{j=1}^{j=m} \beta_j v(i,j) \leq b_i$
\item $\forall i \in \{1,\dots,n\}$,
$\Sigma_{j=1}^{j=m} \beta_j v(i,j) - \Sigma_{j=1}^{j=m} \alpha_j v(i,j) \leq - b_i$

\end{enumerate}

The LP problem consists in finding values of $(\alpha_1,\dots,\alpha_m)$,
$(\beta_1,\dots,\beta_m)$
and $\gamma$ maximizing $- \gamma$ while satisfying the aforementioned inequalities.

\subsection*{Step 6: Reducing AOBM to LP}

Let $(\alpha_1,\dots,\alpha_m)$,
$(\beta_1,\dots,\beta_m)$
and $\gamma$ be a solution to the LP problem.
$\forall x \in S$, let $\lambda(x)$ be the integer $j$ such that $x \in S_j$.
Let $u$ be the function from $S$ to $\mathbb{R}$ such that,
$\forall x \in S$,
$u(x) = \alpha_{\lambda(x)} - \beta_{\lambda(x)}$.

Theorem~\ref{thlin} below says that $u$ solves the AOBM problem. We first prove some lemmas.

\begin{lemma}
\label{lemmax}
$|| u ||_\infty \geq \max(\alpha^*,\beta^*)$,
where
$\alpha^* = \max_{j \in \{1,\dots,m\}} \alpha_j$
and
$\beta^* = \max_{j \in \{1,\dots,m\}} \beta_j$.
\end{lemma}

\begin{proof}
Suppose the opposite:
$|| u ||_\infty < \max(\alpha^*,\beta^*)$.
According to inequalities 2, $\gamma \geq \max(\alpha^*,\beta^*)$. Thus, $\gamma > || u ||_\infty$.

$\forall j \in \{1,\dots,m\}$,
we define $(\alpha'_1,\dots,\alpha'_m)$ and
$(\beta'_1,\dots,\beta'_m)$
as follows:

\begin{itemize}

\item If $\alpha_j \geq \beta_j$,
$\alpha'_j = \alpha_j - \beta_j$
and $\beta'_j = 0$.

\item Otherwise, $\alpha'_j = 0$ and
$\beta'_j = \beta_j - \alpha_j$.

\end{itemize}

Let $\gamma' = || u ||_\infty < \gamma$.

$\forall j \in \{1,\dots,m\}$,
$\alpha'_j \leq \max(\alpha_j,\beta_j)$ and $\beta'_j \leq \max(\alpha_j,\beta_j)$.
Thus, $\alpha'_j \leq \gamma'$
and $\beta'_j \leq \gamma'$.

We now show that
$(\alpha'_1,\dots,\alpha'_m)$,
$(\beta'_1,\dots,\beta'_m)$
and $\gamma'$ satisfy the inequalities of the LP problem.

Inequalities 1 are satisfied by definition.
$\forall j \in \{1,\dots,m\}$,
$\alpha'_j \leq \gamma'$
and $\beta'_j \leq \gamma'$.
Thus, $\alpha'_j - \gamma' \leq 0$
and $\beta'_j - \gamma' \leq 0$,
and inequalities 2 are satisfied.

Inequalities 3 and 4 are equivalent to:
$\forall i \in \{1,\dots,n\}$,
$\Sigma_{j=1}^{j=m} \alpha_j v(i,j) - \Sigma_{j=1}^{j=m} \beta_j v(i,j) = b_i$.
$\forall j \in \{1,\dots,m\}$:

\begin{itemize}
\item If $\alpha_j \geq \beta_j$,
$\alpha'_j - \beta'_j = (\alpha_j - \beta_j) - 0 = \alpha_j - \beta_j$.
\item Otherwise, 
$\alpha'_j - \beta'_j = 0 - (\beta_j - \alpha_j) = \alpha_j - \beta_j$.
\end{itemize}

Thus, $\forall j \in \{1,\dots,m\}$,
$\alpha'_j - \beta'_j = \alpha_j - \beta_j$.
Thus,
$\Sigma_{j=1}^{j=m} \alpha'_j v(i,j) - \Sigma_{j=1}^{j=m} \beta'_j v(i,j)
= \Sigma_{j=1}^{j=m} (\alpha'_j - \beta'_j) v(i,j)
= \Sigma_{j=1}^{j=m} (\alpha_j - \beta_j) v(i,j)
= \Sigma_{j=1}^{j=m} \alpha_j v(i,j) - \Sigma_{j=1}^{j=m} \beta_j v(i,j) = b_i$.
Thus, inequalities 3 and 4 are satisfied.

Thus, there exists
$(\alpha'_1,\dots,\alpha'_m)$,
$(\beta'_1,\dots,\beta'_m)$
and $\gamma'$
satisfying the inequalities of the LP problem
with $- \gamma' > - \gamma$.
Thus, $(\alpha_1,\dots,\alpha_m)$,
$(\beta_1,\dots,\beta_m)$
and $\gamma$ do not solve the LP problem: contradiction. Thus, the result.
\end{proof}

\begin{lemma}
\label{lemgamma}
$|| u ||_\infty = \gamma$.
\end{lemma}

\begin{proof}
$|| u ||_\infty = \max_{x \in S} |u(x)| = \max_{j \in \{1,\dots,m\}} |\alpha_j - \beta_j|$.
$\forall j \in \{1,\dots,m\}$,
$\alpha_j \leq \gamma$
and
$\beta_j \leq \gamma$.
Thus, $|\alpha_j - \beta_j| \leq \gamma$, and $|| u ||_\infty \leq \gamma$.

We now show that $\gamma \leq || u ||_\infty$.
Suppose the opposite: $\gamma > || u ||_\infty$. As the LP problem consists in maximizing $-\gamma$ (and thus, minimizing $\gamma$),
this implies that the inequalities of the LP problem are not compatible with
$\gamma \leq || u ||_\infty$.
Variable $\gamma$ only appears in inequalities 1 and 2, and these inequalities impose to have
$\gamma \geq 0$,
$\gamma \geq \max_{j \in \{1,\dots,m\}} \alpha_j$
and
$\gamma \geq \max_{j \in \{1,\dots,m\}} \beta_j$.
Thus,
$\gamma = \max(a^*,b^*)$,
where
$\alpha^* = \max_{j \in \{1,\dots,m\}} \alpha_j$
and
$\beta^* = \max_{j \in \{1,\dots,m\}} \beta_j$.
Thus, according to Lemma~\ref{lemmax},
$|| u ||_\infty \geq \gamma$.

Therefore, $|| u ||_\infty = \gamma$.
\end{proof}

\begin{theorem}
\label{thlin}
Function $u$ solves the AOBM problem.
\end{theorem}

\begin{proof}
By definition, $u \in Z$.

Inequalities 3 and 4 of the LP problem are equivalent to:
$\forall i \in \{1,\dots,n\}$,
$\Sigma_{j=1}^{j=m} \alpha_j v(i,j) - \Sigma_{j=1}^{j=m} \beta_j v(i,j) = b_i$.
Thus, $\forall i \in \{1,\dots,n\}$,
$b_i = \Sigma_{j=1}^{j=m} (\alpha_j - \beta_j) v(i,j) = \Sigma_{j=1}^{j=m} (\alpha_j - \beta_j) \int_{x \in S_j} p_i(x) dx =
\Sigma_{j=1}^{j=m} \int_{x \in S_j} u(x) p_i(x) dx = \int_{x \in S} p_i(x) u(x) dx$.
Thus, $u \in \Omega$.

Therefore, $u \in \Omega \cap Z$.
Now, suppose the opposite of the claim:
$u \notin \arg \min_{g \in \Omega \cap Z} || g ||_\infty$.
Let $w \in \arg \min_{g \in \Omega \cap Z} || g ||_\infty$.
Thus, $|| w ||_\infty < || u ||_\infty$.

Let $(w_1,\dots,w_m)$ be such that,
$\forall x \in S_j$,
$w(x) = w_j$.
Let $\gamma' = || w ||_\infty$.
$\forall j \in \{1,\dots,m\}$,
we define
$(\alpha_1',\dots,\alpha_m')$
and
$(\beta_1',\dots,\beta_m')$
as follows:

\begin{itemize}

\item If $w_j \geq 0$, $\alpha'_j = w_j$
and $\beta'_j = 0$.

\item Otherwise, $\alpha'_j = 0$
and $\beta'_j = -w_j$.
\end{itemize}

Thus, inequalities 1 are satisfied.

As $\gamma' = || w ||_\infty$,
$\forall j \in \{1,\dots,m\}$,
$\gamma' \geq |w_j| \geq \max(\alpha'_j,\beta'_j)$.
Thus, inequalities 2 are satisfied.

As $w \in \Omega$,
$\forall i \in \{1,\dots,n\}$,
$\int_{x \in S_j} p_i(x) w(x) dx = b_i$.
Thus,
$\forall i \in \{1,\dots,n\}$,
$b_i = \Sigma_{j=1}^{j=m} \int_{x \in S_j} p_i(x) w(x) dx$ $=
\Sigma_{j=1}^{j=m} w_j \int_{x \in S_j} p_i(x) dx =
\Sigma_{j=1}^{j=m} w_j v(i,j) =
\Sigma_{j=1}^{j=m} \alpha'_j v(i,j) - \Sigma_{j=1}^{j=m} \beta'_j v(i,j) \leq b_i$.
Thus, inequalities 3 and 4 are satisfied.

According to Lemma~\ref{lemgamma},
$|| u ||_\infty = \gamma$.
Thus, as $|| w ||_\infty < || u ||_\infty$,
$\gamma' < \gamma$.
Therefore, there exists
$(\alpha'_1,\dots,\alpha'_m)$,
$(\beta'_1,\dots,\beta'_m)$
and $\gamma'$
satisfying the inequalities of the LP problem
with $- \gamma' > - \gamma$.
Thus, $(\alpha_1,\dots,\alpha_m)$,
$(\beta_1,\dots,\beta_m)$
and $\gamma$ do not solve the LP problem: contradiction. Thus, the result.
\end{proof}

\section{The Inverse Case}
\label{sec_altpb}

In the previous section, we showed how to reach the desired scores for each population with a minimal individual error. However, even when minimized, the individual error may still be very high, and sometimes not acceptable.

In this section, we consider the inverse problem: assuming that we can accept an individual error which is at most $\epsilon$, how can we reach a score which is \emph{as close as possible} from the desired scores of each population? We call this problem the inverse ODR (IODR) problem.

We again proceed in 6 steps, following the same outline as the 6 steps of Section~\ref{sec_gen}.

\subsection*{Step 1: The Inverse ODR (IODR) Problem}

Let $\epsilon \geq 0$.
Let $\Phi_0$ be the set of functions $g$ from $S$ to $\mathbb{R}$ such that $||g - f||_\infty \leq \epsilon$ (i.e., functions for which the individual error remains acceptable).

Let $g$ be a function from $S$ to $\mathbb{R}$.
$\forall i \in \{1,\dots,n\}$, let
$\mu_i(g) = |\int_{x \in S} p_i(x) g(x) dx - y_i |$ (i.e., the distance between the average score of population $i$ and its desired average score $y_i$).
Let $\mu(g) = \max_{i \in \{1,\dots,n\}} \mu_i(g)$ (i.e., the upper bound of these distances).

The IODR problem consists in finding a function
$h \in \arg\min_{g \in \Phi_0} \mu(g)$.

\subsection*{Step 2: The Inverse OBM (IOBM) Problem}

Let $\Phi$ be the set of functions $g$ from $S$ to $\mathbb{R}$
such that $|| g ||_\infty  \leq \epsilon$.

$\forall i \in \{1,\dots,n\}$, let
$\Delta_i(g) = | \int_{x \in S} p_i(x) g(x) dx - b_i |$.
Let $\Delta(g) = \max_{i \in \{1,\dots,n\}} \Delta_i (g)$.

The IOBM problem consists in finding a function $u \in \arg\min_{g \in \Phi} \Delta(g)$.

\subsection*{Step 3: Reducing IODR to IOBM}

In Theorem~\ref{iobm_eq}, we show that a solution to the IOBM problem provides an immediate solution to the IODR problem.

\begin{theorem}
\label{iobm_eq}
If $u$ solves the IOBM problem, then $u + f$ solves the IODR problem.
\end{theorem}

\begin{proof}
Let $g$ be a function from $S$ to $\mathbb{R}$.
$\forall i \in \{1,\dots,n\}$,
$\Delta(u) = |\int_{x \in S} p_i(x) g(x) dx - b_i|
= |\int_{x \in S} p_i(x) g(x) dx + \int_{x \in S} p_i(x) f(x) dx - \int_{x \in S} p_i(x) f(x) dx - b_i| 
= |\int_{x \in S} p_i(x) (g(x) + f(x)) dx - y_i|
= \mu_i(g+f)$.
Thus, $\Delta(g) = \mu(g+f)$,
and $\arg\min_{g \in \Phi} \Delta(g) = \arg\min_{g \in \Phi_0} \mu(g+f)$.

Therefore, if $u \in \arg\min_{g \in \Phi} \Delta(g)$, then 
$u + f \in \arg\min_{g \in \Phi_0} \mu(g)$. Thus, the result
\end{proof}

\subsection*{Step 4: The Inverse AOBM (IAOBM) Problem}

The IAOBM problem consists in finding a function $u \in \arg\min_{g \in \Phi \cap Z} \Delta(g)$.

\subsection*{Step 5: The Inverse LP (ILP) Problem}

Let $(\alpha_1,\dots,\alpha_m)$,
$(\beta_1,\dots,\beta_m)$
and $\gamma$ be $2m+1$ variables.

Consider the following inequalities:

\begin{enumerate}

\item $\gamma \geq 0$,
and $\forall j \in \{1,\dots,m\}$,
$\alpha_j \geq 0$ and
$\beta_j \geq 0$.

\item $\forall j \in \{1,\dots,m\}$,
$\alpha_j \leq \epsilon$ and $\beta_j \leq \epsilon$.

\item $\forall i \in \{1,\dots,n\}$,
$\Sigma_{j=1}^{j=m} \alpha_j  v(i,j) - \Sigma_{j=1}^{j=m} \beta_j  v(i,j) - b_i \leq \gamma$

\item $\forall i \in \{1,\dots,n\}$,
$\Sigma_{j=1}^{j=m} \beta_j  v(i,j) - \Sigma_{j=1}^{j=m} \alpha_j  v(i,j) + b_i \leq \gamma$

\end{enumerate}

The ILP problem consists in finding values of $(\alpha_1,\dots,\alpha_m)$,
$(\beta_1,\dots,\beta_m)$
and $\gamma$ maximizing $- \gamma$ while satisfying the aforementioned inequalities.

\subsection*{Step 6: Reducing IAOBM to ILP}

Let $(\alpha_1,\dots,\alpha_m)$,
$(\beta_1,\dots,\beta_m)$
and $\gamma$ be a solution to the ILP problem.
Let $u$ be the function from $S$ to $\mathbb{R}$ such that,
$\forall x \in S$,
$u(x) = \alpha_{\lambda(x)} - \beta_{\lambda(x)}$.

In Theorem~\ref{ilp_eq}, we show that $u$ solves the IAOBM problem.

\begin{lemma}
\label{lem_delta_1}
$\Delta(u) \leq \gamma$.
\end{lemma}

\begin{proof}
$\forall i \in \{1,\dots,n\}$,
$\Delta_i(u) = |\int_{x \in S} p_i(x) u(x) dx - b_i|
= |\Sigma_{j=1}^{j=m} \int_{x \in S_j} p_i(x) u(x) dx - b_i|
= |\Sigma_{j=1}^{j=m} (\alpha_j - \beta_j) v(i,j) dx - b_i|
= |\Sigma_{j=1}^{j=m} \alpha_j v(i,j) - \Sigma_{j=1}^{j=m} \beta_j v(i,j) - bi| \leq \gamma$,
according to inequalities 3 and 4.
Thus, $\Delta(u) = \max_{i \in \{1,\dots,n\}} \Delta_i(u) \leq \gamma$.\end{proof}

\begin{lemma}
\label{lem_delta_2}
$\Delta(u) \geq \gamma$.
\end{lemma}

\begin{proof}
Suppose the opposite:
$\Delta(u) < \gamma$.
Let $\gamma' = \Delta(u)$.
$\forall i \in \{1,\dots,n\}$,
$\Delta_i(u) = | \int_{x \in S} p_i(x) u(x) dx - b_i |
= |\Sigma_{j=1}^{j=m} \alpha_j  v(i,j) - \Sigma_{j=1}^{j=m} \beta_j  v(i,j) - b_i| \leq \Delta(u)$.
Thus, as $\gamma' = \Delta(u)$, inequalities 3 and 4 are still satisfied if we replace $\gamma$ by $\gamma'$.
Thus, as $\gamma' < \gamma$,
$(\alpha_1,\dots,\alpha_m)$,
$(\beta_1,\dots,\beta_m)$
and $\gamma$ do not solve the ILP problem: contradiction. Thus, the result.
\end{proof}

\begin{theorem}
\label{ilp_eq}
The function $u$ solves the IAOBM problem.
\end{theorem}

\begin{proof}
By definition, $u \in Z$.
According to inequalities 2, $u \in \Phi$.
Thus, $u \in \Phi \cap Z$.

Now, suppose the opposite of the claim:
$u \notin \arg\min_{g \in \Phi \cap Z} \Delta(g)$.
Let $w \in \arg\min_{g \in \Phi \cap Z} \Delta(g)$.

Let $(w_1,\dots,w_m)$ be such that,
$\forall x \in S_j$,
$w(x) = w_j$.
Let $\gamma' = \Delta(w)$.
$\forall j \in \{1,\dots,m\}$,
we define
$(\alpha_1',\dots,\alpha_m')$
and
$(\beta_1',\dots,\beta_m')$
as follows:

\begin{itemize}
\item If $w_j \geq 0$, $\alpha'_j = w_j$
and $\beta'_j = 0$.

\item Otherwise, $\alpha'_j = 0$
and $\beta'_j = -w_j$.
\end{itemize}

By construction, inequalities 1 are satisfied.

As $w \in \Phi$,
$\forall j \in \{1,\dots,m\}$,
$|w_j| \leq \epsilon$.
Thus, $\forall j \in \{1,\dots,m\}$,
$\alpha'_j \leq |w_j| \leq \epsilon$
and
$\beta'_j \leq |w_j| \leq \epsilon$.
Therefore, inequalities 2 are satisfied.

As $\gamma = \Delta(w)$,
$\forall i \in \{1,\dots,n\}$,
$| \int_{x \in S} p_i(x) w(x) dx - b_i |
= |\Sigma_{j=1}^{j=m} \alpha'_j  v(i,j) - \Sigma_{j=1}^{j=m} \beta'_j  v(i,j) - b_i| \leq \Delta(w) = \gamma'$.
Thus, inequalities 3 and 4 are satisfied.

As $w \in \arg\min_{g \in \Phi \cap Z} \Delta(g)$ and
$u \notin \arg\min_{g \in \Phi \cap Z} \Delta(g)$,
we have $\Delta(w) < \Delta(u)$.
We have $\Delta(w) = \gamma'$,
and according to Lemma~\ref{lem_delta_1}
and Lemma~\ref{lem_delta_2},
$\Delta(u) = \gamma$.
Thus, $\gamma' < \gamma$.
Thus,
$(\alpha_1,\dots,\alpha_m)$,
$(\beta_1,\dots,\beta_m)$
and $\gamma$
do not solve the ILP problem: contradiction.
Thus, the result.
\end{proof}

\section{Conclusion}

\label{sec_conc}

We consider the problem of removing algorithmic discrimination between several populations with a minimal individual error. We first describe an analytical solution to this problem in the case of two populations. We then show that the general case (with $n$ populations) can be solved approximately with linear programming. We also consider the inverse problem where an upper bound on the error is fixed and we seek to minimize the discrimination.

A major challenge would be to either find an analytical solution to the general case with $n$ populations or prove that it is indeed intractable. We conjecture the latter.
Another interesting question would be to determine how to optimally choose the subsets $(S_1,\dots,S_m)$ used for the approximate solution.

\bibliographystyle{plain}
\bibliography{biblio}

\end{document}